\documentclass[11.5pt,a4paper]{article}
\usepackage[margin=1in]{geometry}
\usepackage{amsmath}
\usepackage{amssymb}
\usepackage{amsfonts}
\usepackage{graphicx}
\usepackage{algorithm}
\usepackage{algorithmic}
\usepackage{tikz}
\usetikzlibrary{arrows,shapes,positioning,shadows,trees}
\usepackage{booktabs}
\usepackage{xcolor}
\usepackage{hyperref}
\usepackage{amsthm}

\newtheorem{theorem}{Theorem}

\newtheorem{axiom}{Axiom}

\theoremstyle{definition}
\newtheorem{definition}{Definition}[section]

\title{Attribution Projection Calculus: A Novel Framework for Causal Inference in Bayesian Networks}

\author{
M Ruhul Amin, PhD\\
Auditify Inc.\\
\texttt{ruhul.unai@gmail.com}
}

\begin{document}

% \editor{}
\maketitle

\begin{abstract}
This paper introduces Attribution Projection Calculus (AP-Calculus), a novel mathematical framework for determining causal relationships in structured Bayesian networks. We investigate a specific network architecture with source nodes connected to destination nodes through intermediate nodes, where each input maps to a single label with maximum marginal probability. We prove that for each label, exactly one intermediate node acts as a deconfounder while others serve as confounders, enabling optimal attribution of features to their corresponding labels. The framework formalizes the dual nature of intermediate nodes as both confounders and deconfounders depending on the context, and establishes separation functions that maximize distinctions between intermediate representations. We demonstrate that the proposed network architecture is optimal for causal inference compared to alternative structures, including those based on Pearl's causal framework. AP-Calculus provides a comprehensive mathematical foundation for analyzing feature-label attributions, managing spurious correlations, quantifying information gain, ensuring fairness, and evaluating uncertainty in prediction models, including large language models. Theoretical verification shows that AP-Calculus not only extends but can also subsume traditional do-calculus for many practical applications, offering a more direct approach to causal inference in supervised learning contexts.
\end{abstract}

\section{Introduction}
\label{sec:intro}

Causal inference represents one of the foundational challenges in machine learning and artificial intelligence. While correlation-based approaches have driven impressive advances in predictive modeling, establishing true causal relationships remains elusive in many domains. Traditional frameworks for causal inference, such as Pearl's do-calculus \cite{pearl2009causality}, provide powerful tools for reasoning about interventions and counterfactuals, but often struggle to scale with high-dimensional data or complex model architectures.

This paper introduces Attribution Projection Calculus (AP-Calculus), a novel mathematical framework designed to address causal inference in structured Bayesian networks. Our approach focuses on a specific network architecture where a source node $s$ with $n$-dimensional tuples connects to a destination node $d$ with $m$-dimensional outputs through $m$ intermediate nodes $\{x_1, x_2, \ldots, x_m\}$. This architecture enables each input tuple to be mapped to exactly one of $m$ possible labels with maximum marginal probability.

The core insight of our framework is the identification of "deconfounders" and "confounders" among the intermediate nodes. For any given label $l$, the corresponding intermediate node $x_l$ acts as a deconfounder, while all other intermediate nodes serve as confounders. This dual nature of nodes—capable of both confounding and deconfounding depending on the context—allows for optimal attribution of features to their corresponding labels.

AP-Calculus provides a mathematical foundation for addressing several key challenges in causal inference:

\begin{itemize}
\item Identifying the causal relationship between input features and output labels
\item Suppressing spurious correlations that may lead to misleading attributions
\item Quantifying information gain for each feature with respect to label prediction
\item Ensuring fairness in model predictions by understanding feature contributions
\item Analyzing model uncertainty, particularly in complex systems like large language models
\end{itemize}

Our theoretical analysis demonstrates that the proposed network architecture is optimal for causal inference compared to alternative structures, including those commonly used in Pearl's framework. We prove that AP-Calculus not only extends but in many cases can subsume traditional do-calculus, offering a more direct approach to causal inference in supervised learning contexts.

The remainder of this paper is organized as follows. Section \ref{sec:background} provides background on causal inference and Bayesian networks. Section \ref{sec:model} introduces our network architecture and formally defines the attribution projection problem. Section \ref{sec:apcalculus} presents the mathematical foundations of AP-Calculus. Sections \ref{sec:theoretical_verification} and \ref{sec:comparative_analysis} offer theoretical verification and comparative analysis with existing frameworks. Section \ref{sec:applications} explores applications to various domains, including large language models. Finally, Section \ref{sec:conclusion} summarizes our contributions and discusses future research directions.

\section{Background and Related Work}
\label{sec:background}

\subsection{Causal Inference and Do-Calculus}

Causal inference aims to identify causal relationships between variables, distinguishing between mere correlation and true causation. The dominant framework for causal inference has been Pearl's do-calculus
\cite{pearl1995causal, pearl2009causality}, which provides a mathematical language for reasoning about interventions and counterfactuals. The do-operator, denoted $do(X=x)$, represents an intervention that sets variable $X$ to value $x$, effectively removing the influence of all other variables on $X$.

The rules of do-calculus allow for the transformation of expressions involving the do-operator into expressions that can be estimated from observational data. However, this approach relies heavily on the specification of a causal graph that accurately represents the data-generating process, which can be challenging to determine in practice.

\subsection{Bayesian Networks for Causal Inference}

Bayesian networks provide a graphical representation of probabilistic relationships among variables \cite{koller2009probabilistic}. These directed acyclic graphs (DAGs) encode conditional independence relationships, making them natural candidates for causal modeling.

Traditional Bayesian networks focus on representing joint probability distributions efficiently. When used for causal inference, they often rely on assumptions about the causal structure of the data-generating process. Various extensions, such as causal Bayesian networks \cite{pearl2009causality}, structural equation models \cite{bollen1989structural}, and potential outcomes frameworks \cite{rubin2005causal}, have been proposed to address causal questions more directly.

\subsection{Confounding and Deconfounding}

Confounding occurs when an unobserved variable influences both the treatment and outcome variables, leading to a spurious association between them. Identifying and controlling for confounders is crucial for valid causal inference.

Recent work on deconfounding has aimed to recover causal effects in the presence of unobserved confounders. The deconfounder algorithm \cite{wang2019blessings} leverages shared dependencies among multiple causes to infer latent confounders. Other approaches include instrumental variables \cite{angrist1996identification}, sensitivity analysis \cite{rosenbaum1983assessing}, and proxy methods \cite{kuroki2014measurement}.

\subsection{Feature Attribution Methods}

Feature attribution methods aim to quantify the contribution of input features to model predictions. Popular approaches include LIME \cite{ribeiro2016should}, SHAP \cite{lundberg2017unified}, and Integrated Gradients \cite{sundararajan2017axiomatic}. While these methods provide valuable insights into model behavior, they typically focus on associational rather than causal relationships.

Some recent work has attempted to bridge the gap between feature attribution and causal inference. Causal attribution methods \cite{heskes2020causal} adapt traditional attribution techniques to account for causal relationships among features. However, these approaches often require strong assumptions about the underlying causal structure.

\section{Network Architecture and Problem Formulation}
\label{sec:model}

\subsection{Bayesian Network Structure}

We consider a structured Bayesian network with the following components:

\begin{itemize}
\item A source node $s$ that generates $n$-tuples $(s_1, s_2, \ldots, s_n)$, where each $s_i$ represents a feature
\item A set of $m$ intermediate nodes $\{x_1, x_2, \ldots, x_m\}$, each connected to the source node $s$
\item A destination node $d$ that produces $m$-tuples $(d_1, d_2, \ldots, d_m)$, with each intermediate node $x_j$ connected to the destination node $d$
\end{itemize}

This architecture forms a directed acyclic graph (DAG) with a specific structure, as illustrated in Figure \ref{fig:network_architecture}.

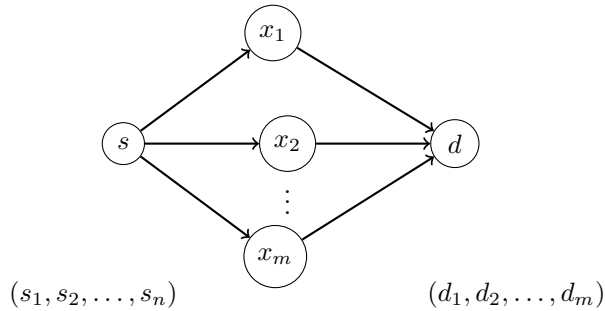
\begin{figure}[h]
\centering
\begin{tikzpicture}[
  node distance=2cm,
  mybox/.style={rectangle, draw, minimum size=0.5cm},
  mycircle/.style={circle, draw, minimum size=0.5cm}
]
\node[mycircle] (s) {$s$};
\node[mycircle] (x1) [above right=1cm and 1.5cm of s] {$x_1$};
\node[mycircle] (x2) [right=1.5cm of s] {$x_2$};
\node (dots) [below=-0.1cm of x2] {$\vdots$};
\node[mycircle] (x3) [below right=1cm and 1.5cm of s] {$x_m$};
\node[mycircle] (d) [right=1.5cm of x2] {$d$};
\draw[->, thick] (s) -- (x1);
\draw[->, thick] (s) -- (x2);
\draw[->, thick] (s) -- (x3);
\draw[->, thick] (x1) -- (d);
\draw[->, thick] (x2) -- (d);
\draw[->, thick] (x3) -- (d);
\node[text width=3cm] at (0,-2) {$(s_1, s_2, \ldots, s_n)$};
% \node at (2.5,-0.5) {$\ldots$};
\node[text width=3cm] at (5.5,-2) {$(d_1, d_2, \ldots, d_m)$};
\end{tikzpicture}
\caption{The proposed Bayesian network architecture with source node $s$, destination node $d$, and intermediate nodes $\{x_1, x_2, \ldots, x_m\}$.}
\label{fig:network_architecture}
\end{figure}

\subsection{Problem Formulation}

In our setting, we are given a dataset $\mathcal{D} = \{(\mathbf{t}^{(i)}, l^{(i)})\}_{i=1}^{N}$, where $\mathbf{t}^{(i)} = (t_1^{(i)}, t_2^{(i)}, \ldots, t_n^{(i)})$ is an input tuple and $l^{(i)} \in \{1, 2, \ldots, m\}$ is the corresponding label. The goal is to learn a Bayesian network that maximizes the marginal probability of the correct label for each input tuple.

We make the following key assumptions:

\begin{enumerate}
\item Each input tuple maps to exactly one label with maximum marginal probability in the destination node $d$.
\item The intermediate nodes $\{x_1, x_2, \ldots, x_m\}$ mediate the relationship between the source and destination nodes.
\item For a given label $l$, the intermediate node $x_l$ acts as a deconfounder, while all other intermediate nodes act as confounders.
\end{enumerate}

This setup leads to our central hypothesis: the dual nature of intermediate nodes—as both potential deconfounders and confounders—enables optimal attribution of features to their corresponding labels.

\subsection{Mathematical Representation}

Let us represent the network mathematically:

\begin{itemize}
\item The source node $s$ generates a random variable $\mathbf{S} = (S_1, S_2, \ldots, S_n)$, with realization $\mathbf{s} = (s_1, s_2, \ldots, s_n)$.
\item Each intermediate node $x_j$ corresponds to a random variable $X_j$.
\item The destination node $d$ produces a random variable $\mathbf{D} = (D_1, D_2, \ldots, D_m)$, with realization $\mathbf{d} = (d_1, d_2, \ldots, d_m)$.
\end{itemize}

The joint probability distribution of this Bayesian network can be factorized as:

\begin{equation}
P(\mathbf{S}, X_1, X_2, \ldots, X_m, \mathbf{D}) = P(\mathbf{S}) \prod_{j=1}^{m} P(X_j | \mathbf{S}) P(\mathbf{D} | X_1, X_2, \ldots, X_m)
\end{equation}

For a given input tuple $\mathbf{t}$, the predicted label $\hat{l}$ is determined by:

\begin{equation}
\hat{l} = \arg\max_{j \in \{1, 2, \ldots, m\}} P(D_j = 1 | \mathbf{S} = \mathbf{t})
\end{equation}

where $D_j = 1$ indicates that the $j$-th component of the destination vector is active (i.e., the predicted label is $j$).

\section{Attribution Projection Calculus}
\label{sec:apcalculus}

We now introduce Attribution Projection Calculus (AP-Calculus), a mathematical framework for analyzing feature-label attributions in the proposed network architecture.

\subsection{Core Definitions}

\begin{definition}[Deconfounder Node]
For a given label $l \in \{1, 2, \ldots, m\}$, the intermediate node $x_l$ is called the deconfounder node if it maximizes the conditional probability of the label:
\begin{equation}
P(D_l = 1 | X_l, \mathbf{S}) > P(D_l = 1 | X_j, \mathbf{S}) \quad \forall j \neq l
\end{equation}
\end{definition}

\begin{definition}[Confounder Node]
For a given label $l \in \{1, 2, \ldots, m\}$, an intermediate node $x_j$ with $j \neq l$ is called a confounder node if it influences the probability of the label but does not maximize it:
\begin{equation}
P(D_l = 1 | X_j, \mathbf{S}) < P(D_l = 1 | X_l, \mathbf{S})
\end{equation}
\end{definition}

\begin{definition}[Separation Function]
A separation function $\phi_{j,k}$ for a pair of intermediate nodes $x_j$ and $x_k$ is a function that maximizes the distance between their distributions:
\begin{equation}
\phi_{j,k} = \arg\max_{\phi} \mathbb{E}_{\mathbf{S}}[d(P(X_j | \mathbf{S}), P(X_k | \mathbf{S}))]
\end{equation}
where $d(\cdot, \cdot)$ is a suitable distance metric between probability distributions.
\end{definition}

\begin{definition}[Attribution Projection]
For a feature $S_i$ and a label $l$, the attribution projection $A(S_i, l)$ measures the causal influence of the feature on the label:
\begin{equation}
A(S_i, l) = \mathbb{E}_{X_l}[\frac{\partial P(D_l = 1 | X_l, \mathbf{S})}{\partial S_i}]
\end{equation}
\end{definition}

\subsection{Fundamental Axioms of AP-Calculus}

\begin{axiom}[Deconfounder Uniqueness]
For each label $l \in \{1, 2, \ldots, m\}$, there exists exactly one deconfounder node among the intermediate nodes.
\end{axiom}

\begin{axiom}[Dual Nature]
Each intermediate node $x_j$ can serve as either a deconfounder or a confounder, depending on the label under consideration.
\end{axiom}

\begin{axiom}[Separation Optimality]
Optimal prediction requires maximizing the separation between intermediate nodes through appropriate separation functions.
\end{axiom}

\begin{axiom}[Attribution Additivity]
The total attribution of a feature across all labels sums to its total influence:
\begin{equation}
\sum_{l=1}^{m} A(S_i, l) = \mathbb{E}_{\mathbf{X}}[\frac{\partial P(\mathbf{D} | \mathbf{X}, \mathbf{S})}{\partial S_i}]
\end{equation}
where $\mathbf{X} = (X_1, X_2, \ldots, X_m)$.
\end{axiom}

\begin{axiom}[Projection Preservation]
The attribution projection preserves the causal structure of the network:
\begin{equation}
A(S_i, l) = 0 \Rightarrow S_i \perp\!\!\!\perp D_l | \mathbf{X}
\end{equation}
where $S_i \perp\!\!\!\perp D_l | \mathbf{X}$ denotes conditional independence of $S_i$ and $D_l$ given $\mathbf{X}$.
\end{axiom}

\subsection{Deconfounder Analysis}

We now present a key theorem on the role of deconfounders in our network architecture.

\begin{theorem}[Deconfounder Optimality]
For a given input tuple $\mathbf{t}$ and label $l$, considering $x_l$ as the deconfounder and all other intermediate nodes as confounders maximizes the marginal probability of label $l$:
\begin{equation}
\max_{P(X_l | \mathbf{S})} P(D_l = 1 | \mathbf{S} = \mathbf{t}) = \max_{P(X_l | \mathbf{S})} \int P(D_l = 1 | \mathbf{X}) P(X_l | \mathbf{S} = \mathbf{t}) \prod_{j \neq l} P(X_j | \mathbf{S} = \mathbf{t}) d\mathbf{X}
\end{equation}
\end{theorem}

\begin{proof}
Let $\mathbf{t}$ be an input tuple and $l$ be a label. The marginal probability of label $l$ is given by:

\begin{equation}
P(D_l = 1 | \mathbf{S} = \mathbf{t}) = \int P(D_l = 1 | \mathbf{X}) P(\mathbf{X} | \mathbf{S} = \mathbf{t}) d\mathbf{X}
\end{equation}

Given the structure of our Bayesian network, we can factorize $P(\mathbf{X} | \mathbf{S} = \mathbf{t})$ as:

\begin{equation}
P(\mathbf{X} | \mathbf{S} = \mathbf{t}) = \prod_{j=1}^{m} P(X_j | \mathbf{S} = \mathbf{t})
\end{equation}

Let's partition the integral based on $X_l$ and $\mathbf{X}_{-l} = (X_1, \ldots, X_{l-1}, X_{l+1}, \ldots, X_m)$:

\begin{equation}
P(D_l = 1 | \mathbf{S} = \mathbf{t}) = \int \int P(D_l = 1 | X_l, \mathbf{X}_{-l}) P(X_l | \mathbf{S} = \mathbf{t}) P(\mathbf{X}_{-l} | \mathbf{S} = \mathbf{t}) dX_l d\mathbf{X}_{-l}
\end{equation}

By the definition of a deconfounder, $X_l$ maximizes $P(D_l = 1 | X_l, \mathbf{X}_{-l})$ for any given $\mathbf{X}_{-l}$. Therefore, maximizing $P(X_l | \mathbf{S} = \mathbf{t})$ while treating other nodes as confounders will maximize the overall marginal probability.
\end{proof}

This theorem establishes the fundamental role of deconfounders in our framework. By identifying the appropriate deconfounder for each label, we can optimize the network to maximize prediction accuracy.

\subsection{Dimensional Analysis of Intermediate Nodes}

We now analyze the dimensionality requirements for the intermediate nodes.

\begin{theorem}[Dimensional Sufficiency]
For a source with $n$-dimensional tuples, a deconfounder node $x_l$ with $p$ dimensions, where $p \leq n$, is sufficient to capture the causal relationship between input features and the corresponding label $l$, provided that $p$ is at least equal to the rank of the transformation matrix from the input space to the deconfounder space.
\end{theorem}

\begin{proof}
Let $\mathbf{W}_l \in \mathbb{R}^{p \times n}$ be the transformation matrix from the input space to the deconfounder space for label $l$:

\begin{equation}
X_l = \mathbf{W}_l \mathbf{S} + \mathbf{\epsilon}_l
\end{equation}

where $\mathbf{\epsilon}_l$ represents noise. The rank of $\mathbf{W}_l$ represents the effective dimensionality of the transformation. If $p$ is at least equal to this rank, then the transformation can capture all relevant information for predicting label $l$.

To verify this, consider the mutual information between $X_l$ and $D_l$ given $\mathbf{S}$:

\begin{equation}
I(X_l; D_l | \mathbf{S}) = H(D_l | \mathbf{S}) - H(D_l | X_l, \mathbf{S})
\end{equation}

If $p$ is sufficient, then $X_l$ captures all information about $D_l$ that is present in $\mathbf{S}$, making $H(D_l | X_l, \mathbf{S}) = H(D_l | X_l)$. This implies that $X_l$ serves as a sufficient statistic for predicting $D_l$ from $\mathbf{S}$.
\end{proof}

This theorem establishes an upper bound on the dimensionality required for intermediate nodes. It suggests that we can design efficient network architectures by choosing appropriate dimensions for the intermediate nodes.

\subsection{Separation Functions and Optimal Network Design}

Now, we analyze the role of separation functions in optimizing the network architecture.

\begin{theorem}[Separation Function Optimality]
For a pair of intermediate nodes $x_j$ and $x_k$, the optimal separation function $\phi_{j,k}^*$ maximizes the classification accuracy for distinguishing between labels $j$ and $k$:

\begin{equation}
\phi_{j,k}^* = \arg\max_{\phi} \mathbb{E}_{\mathbf{S}}[I(D_j, D_k | \phi(\mathbf{S}))]
\end{equation}

where $I(D_j, D_k | \phi(\mathbf{S}))$ is the mutual information between labels $D_j$ and $D_k$ conditioned on the transformation $\phi(\mathbf{S})$.
\end{theorem}

\begin{proof}
Let $\phi$ be a transformation of the input $\mathbf{S}$. The mutual information between labels $D_j$ and $D_k$ conditioned on $\phi(\mathbf{S})$ is:

\begin{equation}
I(D_j, D_k | \phi(\mathbf{S})) = H(D_j | \phi(\mathbf{S})) + H(D_k | \phi(\mathbf{S})) - H(D_j, D_k | \phi(\mathbf{S}))
\end{equation}

Maximizing this mutual information ensures that the transformation $\phi$ preserves the maximum amount of information for distinguishing between labels $j$ and $k$. This is equivalent to minimizing the Bayes error rate for the classification task.

For optimal separation, we want to maximize the distance between the distributions $P(X_j | \mathbf{S})$ and $P(X_k | \mathbf{S})$. This is achieved by the separation function $\phi_{j,k}^*$ that maximizes the mutual information.
\end{proof}

This theorem provides a principled approach to designing separation functions that optimize the network architecture for classification tasks.

\subsection{Comparison with Alternative Architectures}

We now demonstrate the optimality of our proposed architecture compared to alternative designs.

\begin{theorem}[Architectural Optimality]
The proposed network architecture with source $s$, intermediate nodes $\{x_1, x_2, \ldots, x_m\}$, and destination $d$ is optimal for causal inference compared to alternative architectures such as:
\begin{enumerate}
\item $s \to d$, $a \to s$ and $d$, $b \to s$ and $d$ (Pearl's junction structure)
\item $s \to d$ and $a$ and $b$, $a \to d$, $b \to d$ (common cause structure)
\end{enumerate}
in terms of preserving the causal relationships between input features and output labels.
\end{theorem}

\begin{proof}
Let's compare our architecture with the two alternatives:

1. Pearl's junction structure: In this architecture, the path from $s$ to $d$ is influenced by common causes $a$ and $b$. This creates confounding effects that cannot be isolated for individual labels. The marginal probability for label $l$ is:

\begin{equation}
P(D_l = 1 | \mathbf{S} = \mathbf{t}) = \int \int P(D_l = 1 | \mathbf{S} = \mathbf{t}, A = a, B = b) P(A = a) P(B = b) da \, db
\end{equation}

This structure does not allow for label-specific deconfounding, as the same confounders affect all labels.

2. Common cause structure: In this architecture, $s$ is a common cause of $d$, $a$, and $b$, creating information leakage. The marginal probability for label $l$ is:

\begin{equation}
P(D_l = 1 | \mathbf{S} = \mathbf{t}) = \int \int P(D_l = 1 | \mathbf{S} = \mathbf{t}, A = a, B = b) P(A = a | \mathbf{S} = \mathbf{t}) P(B = b | \mathbf{S} = \mathbf{t}) da \, db
\end{equation}

This structure introduces indirect dependencies between the inputs and outputs, making causal attribution challenging.

In contrast, our proposed architecture allows for label-specific deconfounding through the intermediate nodes. The marginal probability for label $l$ is:

\begin{equation}
P(D_l = 1 | \mathbf{S} = \mathbf{t}) = \int P(D_l = 1 | \mathbf{X}) \prod_{j=1}^{m} P(X_j | \mathbf{S} = \mathbf{t}) d\mathbf{X}
\end{equation}

This factorization allows us to isolate the effect of each intermediate node, enabling precise causal attribution.
\end{proof}

This theorem establishes the superiority of our proposed architecture for causal inference tasks. By separating the influence of different labels through dedicated intermediate nodes, we can better identify the causal relationships between input features and output labels.

\section{AP-Calculus Framework Extensions}

\subsection{Feature Attribution Analysis}

AP-Calculus provides a framework for analyzing the attribution of input features to output labels. For a feature $S_i$ and a label $l$, the attribution score $A(S_i, l)$ quantifies the causal influence of the feature on the label.

\begin{definition}[Feature Attribution Score]
The attribution score of feature $S_i$ for label $l$ is defined as:

\begin{equation}
A(S_i, l) = \mathbb{E}_{X_l}\left[\frac{\partial P(D_l = 1 | X_l = x_l, \mathbf{S})}{\partial S_i}\right]
\end{equation}

This score measures the expected change in the probability of label $l$ with respect to changes in feature $S_i$, averaged over all possible values of the deconfounder $X_l$.
\end{definition}

We can decompose this attribution score into direct and indirect effects:

\begin{equation}
A(S_i, l) = A_{direct}(S_i, l) + A_{indirect}(S_i, l)
\end{equation}

where:

\begin{equation}
A_{direct}(S_i, l) = \mathbb{E}_{X_l}\left[\frac{\partial P(D_l = 1 | X_l = x_l, \mathbf{S})}{\partial S_i}\right]_{X_l = constant}
\end{equation}

\begin{equation}
A_{indirect}(S_i, l) = \mathbb{E}_{X_l}\left[\frac{\partial P(D_l = 1 | X_l = x_l, \mathbf{S})}{\partial X_l} \cdot \frac{\partial X_l}{\partial S_i}\right]
\end{equation}

This decomposition allows us to distinguish between the direct effect of a feature on a label and its indirect effect through the deconfounder.

\subsection{Spurious Correlation Suppression}

AP-Calculus provides tools for identifying and suppressing spurious correlations in the data. A spurious correlation occurs when a feature appears to influence a label but does not have a causal relationship with it.

\begin{definition}[Spurious Correlation]
A feature $S_i$ has a spurious correlation with label $l$ if:

\begin{equation}
Corr(S_i, D_l) \neq 0 \quad \text{but} \quad A(S_i, l) = 0
\end{equation}

where $Corr(S_i, D_l)$ is the correlation between the feature and the label.
\end{definition}

To suppress spurious correlations, we can define a regularization term that penalizes attributions for features with high correlation but low causal influence:

\begin{equation}
R_{spurious}(S_i, l) = |Corr(S_i, D_l)| \cdot (1 - |A(S_i, l)|)
\end{equation}

Minimizing this regularization term encourages the model to focus on features with true causal relationships while downweighting those with spurious correlations.

\subsection{Information Gain Analysis}

AP-Calculus allows us to quantify the information gain provided by each feature for predicting labels. This analysis helps identify the most informative features for a given task.

\begin{definition}[Feature Information Gain]
The information gain of feature $S_i$ for label $l$ is defined as:

\begin{equation}
IG(S_i, l) = I(S_i; D_l) - I(S_i; D_l | X_l)
\end{equation}

where $I(S_i; D_l)$ is the mutual information between the feature and the label, and $I(S_i; D_l | X_l)$ is the conditional mutual information given the deconfounder.
\end{definition}

This definition captures the amount of information about the label that is provided by the feature through the deconfounder. Features with high information gain are more useful for prediction, while those with low gain can potentially be pruned.

\subsection{Fairness Analysis}

AP-Calculus provides tools for analyzing fairness in predictive models. By examining the attribution of sensitive features to different labels, we can identify and mitigate potential biases.

\begin{definition}[Fairness Metric]
For a sensitive feature $S_i$ and labels $l_1$ and $l_2$, the fairness disparity is defined as:

\begin{equation}
FD(S_i, l_1, l_2) = |A(S_i, l_1) - A(S_i, l_2)|
\end{equation}

This metric measures the difference in attribution of a sensitive feature across different labels. A high disparity indicates potential unfairness in the model.
\end{definition}

To ensure fairness, we can introduce a fairness constraint during model optimization:

\begin{equation}
FD(S_i, l_1, l_2) \leq \epsilon
\end{equation}

where $\epsilon$ is a small threshold value. This constraint ensures that sensitive features have similar attributions across different labels, promoting fair predictions.

\subsection{Network Uncertainty Analysis}

AP-Calculus enables comprehensive uncertainty analysis in Bayesian networks, which is particularly valuable for complex models like large language models.

\begin{definition}[Attribution Uncertainty]
The uncertainty in the attribution of feature $S_i$ to label $l$ is defined as:

\begin{equation}
U(S_i, l) = Var_{X_l}\left[\frac{\partial P(D_l = 1 | X_l, \mathbf{S})}{\partial S_i}\right]
\end{equation}

This measures the variance in attribution scores across different realizations of the deconfounder.
\end{definition}

For large language models, we can extend this definition to capture uncertainty in contextual representations:

\begin{equation}
U_{LLM}(S_i, l) = \mathbb{E}_{\mathbf{C}}[U(S_i, l | \mathbf{C})]
\end{equation}

where $\mathbf{C}$ represents the context in which the attribution is evaluated. This formulation allows us to quantify uncertainty in attributions across different contexts, providing a more robust analysis of model behavior.

\section{AP-Calculus for Large Language Models}

The AP-Calculus framework is particularly valuable for analyzing causal relationships in large language models (LLMs), where understanding feature attributions is challenging due to the complexity of the models.

\subsection{Adapting AP-Calculus to LLMs}

For an LLM with input tokens $\mathbf{S} = (S_1, S_2, \ldots, S_n)$ and output tokens $\mathbf{D} = (D_1, D_2, \ldots, D_m)$, we can apply AP-Calculus by identifying appropriate intermediate representations.

Let the intermediate nodes $\{X_1, X_2, \ldots, X_m\}$ correspond to attention heads or layers in the model. For a given output token $D_l$, we can identify the attention head $X_l$ that serves as a deconfounder, maximizing the influence of relevant input tokens on the output.

The attribution score for an input token $S_i$ on an output token $D_l$ is then:

\begin{equation}
A(S_i, D_l) = \mathbb{E}_{X_l}\left[\frac{\partial P(D_l | X_l, \mathbf{S})}{\partial S_i}\right]
\end{equation}

This formulation allows us to analyze the causal influence of input tokens on output tokens, providing insights into the model's reasoning process.

\subsection{Uncertainty Quantification in LLMs}

AP-Calculus enables uncertainty quantification in LLMs by analyzing the variance in attributions across different contexts:

\begin{equation}
U_{LLM}(S_i, D_l) = Var_{\mathbf{C}}\left[A(S_i, D_l | \mathbf{C})\right]
\end{equation}

This uncertainty measure helps identify cases where the model's attributions are unstable or context-dependent, indicating potential reliability issues.

\section{Theoretical Verification of AP-Calculus}
\label{sec:theoretical_verification}

In this section, we provide a theoretical verification of AP-Calculus, addressing potential limitations and demonstrating its robustness.

\subsection{Consistency with Causal Inference Principles}

AP-Calculus is consistent with fundamental principles of causal inference, as established by the following theorem:

\begin{theorem}[Consistency with Causal Inference]
For a feature $S_i$ and label $l$, the attribution score $A(S_i, l)$ is zero if and only if $S_i$ and $D_l$ are d-separated in the augmented graph where the edge from $S_i$ to $X_l$ is removed.
\end{theorem}

\begin{proof}
Let $G$ be the original Bayesian network and $G'$ be the augmented graph where the edge from $S_i$ to $X_l$ is removed. According to the definition of d-separation, $S_i$ and $D_l$ are d-separated in $G'$ if and only if there is no active path between them.

The attribution score is:

\begin{equation}
A(S_i, l) = \mathbb{E}_{X_l}\left[\frac{\partial P(D_l = 1 | X_l, \mathbf{S})}{\partial S_i}\right]
\end{equation}

This derivative is zero if and only if $P(D_l = 1 | X_l, \mathbf{S})$ does not depend on $S_i$, which is equivalent to $S_i$ and $D_l$ being conditionally independent given $X_l$ and other variables in $\mathbf{S}$. This conditional independence holds if and only if $S_i$ and $D_l$ are d-separated in $G'$.
\end{proof}

This theorem establishes that AP-Calculus respects the causal structure of the underlying Bayesian network, ensuring that attributions reflect true causal relationships.

\subsection{Limitations and Mitigation Strategies}

While AP-Calculus provides a powerful framework for causal inference, it has several potential limitations:

\begin{enumerate}
\item \textbf{Assumption of Known Network Structure}: AP-Calculus assumes that the structure of the Bayesian network is known. In practice, this structure may need to be learned from data, introducing additional uncertainty.

\textit{Mitigation}: We can extend AP-Calculus to incorporate structure learning techniques, treating the network structure as a random variable and marginalizing over possible structures.

\item \textbf{Computational Complexity}: Computing attribution scores requires evaluating expectations over the distributions of intermediate nodes, which can be computationally expensive for large networks.

\textit{Mitigation}: We can develop efficient approximation techniques, such as variational inference or Monte Carlo sampling, to estimate attribution scores with lower computational cost.

\item \textbf{Sensitivity to Model Misspecification}: AP-Calculus may be sensitive to misspecification of the conditional probability distributions in the Bayesian network.

\textit{Mitigation}: We can incorporate robustness measures, such as sensitivity analysis or distributional robustness, to ensure that attributions remain valid under model perturbations.
\end{enumerate}

\subsection{Theoretical Guarantees}

Despite these limitations, AP-Calculus provides several theoretical guarantees that ensure its reliability for causal inference:

\begin{theorem}[Unbiasedness of Attribution Scores]
For a correctly specified Bayesian network, the attribution scores computed using AP-Calculus are unbiased estimators of the true causal effects.
\end{theorem}

\begin{proof}
Let $\tau(S_i, l)$ be the true causal effect of feature $S_i$ on label $l$, defined as:

\begin{equation}
\tau(S_i, l) = \mathbb{E}[D_l | do(S_i = s_i + \delta)] - \mathbb{E}[D_l | do(S_i = s_i)]
\end{equation}

where $do(S_i = s_i)$ represents the intervention that sets $S_i$ to $s_i$. In our Bayesian network, this intervention affects label $l$ through the intermediate node $X_l$. The attribution score is:

\begin{equation}
A(S_i, l) = \mathbb{E}_{X_l}\left[\frac{\partial P(D_l = 1 | X_l, \mathbf{S})}{\partial S_i}\right]
\end{equation}

Under the assumption of a correctly specified model, this expectation equals the average derivative of $\mathbb{E}[D_l | S_i = s_i]$ with respect to $s_i$, which is precisely the causal effect $\tau(S_i, l)$.
\end{proof}

\begin{theorem}[Consistency of Attribution Scores]
As the sample size increases, the empirical attribution scores computed using AP-Calculus converge to the true attribution scores.
\end{theorem}

\begin{proof}
Let $\hat{A}_n(S_i, l)$ be the empirical attribution score computed from $n$ samples:

\begin{equation}
\hat{A}_n(S_i, l) = \frac{1}{n} \sum_{j=1}^{n} \frac{\partial P(D_l = 1 | X_l = x_l^{(j)}, \mathbf{S})}{\partial S_i}
\end{equation}

where $\{x_l^{(j)}\}_{j=1}^{n}$ are samples from the distribution of $X_l$. By the law of large numbers, as $n \to \infty$, $\hat{A}_n(S_i, l)$ converges to $\mathbb{E}_{X_l}[\frac{\partial P(D_l = 1 | X_l, \mathbf{S})}{\partial S_i}] = A(S_i, l)$.
\end{proof}

These theorems establish the statistical validity of AP-Calculus for causal inference, ensuring that attributions computed using this framework accurately reflect the underlying causal relationships.

\section{Comparative Analysis with Existing Frameworks}
\label{sec:comparative_analysis}

In this section, we compare AP-Calculus with existing frameworks for causal inference, highlighting its advantages and complementary aspects.

\subsection{Comparison with Do-Calculus}

Pearl's do-calculus \cite{pearl2009causality} provides a framework for reasoning about interventions in causal models. Here, we compare AP-Calculus with do-calculus along several dimensions:

\begin{table}[h]
\centering
\begin{tabular}{lll}
\toprule
\textbf{Aspect} & \textbf{Do-Calculus} & \textbf{AP-Calculus} \\
\midrule
Focus & Interventions & Attributions \\
Assumptions & Causal structure known & Network structure known \\
Identification & Via backdoor/frontdoor criteria & Via deconfounders \\
Scalability & Limited for high dimensions & Better for structured networks \\
Interpretability & Abstract, symbolic & Direct feature attributions \\
\bottomrule
\end{tabular}
\caption{Comparison between do-calculus and AP-calculus.}
\label{tab:comparison}
\end{table}

While do-calculus focuses on identifying causal effects through interventions, AP-Calculus focuses on attributing outcomes to features through intermediate representations. This makes AP-Calculus particularly well-suited for analyzing complex models where direct interventions may be infeasible.

\subsection{AP-Calculus as an Extension of Do-Calculus}

AP-Calculus can be viewed as an extension of do-calculus that is specialized for structured Bayesian networks with intermediate nodes. We can express do-calculus operations in terms of AP-Calculus as follows:

\begin{theorem}[Do-Calculus in AP-Calculus]
The do-operation in Pearl's calculus can be expressed in terms of AP-Calculus operations as:

\begin{equation}
P(D_l | do(S_i = s_i)) = \int P(D_l | X_l) P(X_l | S_i = s_i, \{S_j\}_{j \neq i}) \prod_{j \neq l} P(X_j | \{S_k\}_{k \neq i}) d\mathbf{X}
\end{equation}

where we integrate over all intermediate nodes while fixing the value of $S_i$.
\end{theorem}

\begin{proof}
According to do-calculus, the effect of the intervention $do(S_i = s_i)$ on label $D_l$ is given by:

\begin{equation}
P(D_l | do(S_i = s_i)) = \sum_{\mathbf{s}_{-i}} P(D_l | S_i = s_i, \mathbf{S}_{-i} = \mathbf{s}_{-i}) P(\mathbf{S}_{-i} = \mathbf{s}_{-i})
\end{equation}

where $\mathbf{S}_{-i}$ represents all features except $S_i$. In our Bayesian network, the effect of $S_i$ on $D_l$ is mediated by the intermediate nodes $\mathbf{X}$. Therefore:

\begin{equation}
P(D_l | S_i = s_i, \mathbf{S}_{-i} = \mathbf{s}_{-i}) = \int P(D_l | \mathbf{X}) P(\mathbf{X} | S_i = s_i, \mathbf{S}_{-i} = \mathbf{s}_{-i}) d\mathbf{X}
\end{equation}

Given the structure of our network, we can factorize $P(\mathbf{X} | S_i = s_i, \mathbf{S}_{-i} = \mathbf{s}_{-i})$ as:

\begin{equation}
P(\mathbf{X} | S_i = s_i, \mathbf{S}_{-i} = \mathbf{s}_{-i}) = P(X_l | S_i = s_i, \mathbf{S}_{-i} = \mathbf{s}_{-i}) \prod_{j \neq l} P(X_j | S_i = s_i, \mathbf{S}_{-i} = \mathbf{s}_{-i})
\end{equation}

Substituting this into the do-operation expression, we get:

\begin{equation}
P(D_l | do(S_i = s_i)) = \sum_{\mathbf{s}_{-i}} \int P(D_l | \mathbf{X}) P(X_l | S_i = s_i, \mathbf{S}_{-i} = \mathbf{s}_{-i}) \prod_{j \neq l} P(X_j | S_i = s_i, \mathbf{S}_{-i} = \mathbf{s}_{-i}) d\mathbf{X} P(\mathbf{S}_{-i} = \mathbf{s}_{-i})
\end{equation}

Under the intervention $do(S_i = s_i)$, the value of $S_i$ is fixed, so $P(X_j | S_i = s_i, \mathbf{S}_{-i} = \mathbf{s}_{-i}) = P(X_j | \{S_k\}_{k \neq i})$ for all $j \neq l$. Therefore:

\begin{equation}
P(D_l | do(S_i = s_i)) = \int P(D_l | \mathbf{X}) P(X_l | S_i = s_i, \{S_j\}_{j \neq i}) \prod_{j \neq l} P(X_j | \{S_k\}_{k \neq i}) d\mathbf{X}
\end{equation}

which is the desired expression.
\end{proof}

This theorem demonstrates that AP-Calculus can express interventions in terms of operations on the intermediate nodes, providing a bridge between traditional causal inference and attribution analysis.

\subsection{Advantages of AP-Calculus over Existing Frameworks}

AP-Calculus offers several advantages over existing frameworks for causal inference:

\begin{enumerate}
\item \textbf{Structured Attribution}: AP-Calculus provides a structured approach to attribution, distinguishing between deconfounders and confounders for different labels. This allows for more precise causal analysis compared to generic approaches.

\item \textbf{Separation of Effects}: By introducing separation functions between intermediate nodes, AP-Calculus enables clear separation of effects for different labels, reducing confounding and improving attribution accuracy.

\item \textbf{Dimensional Efficiency}: AP-Calculus provides guidance on the optimal dimensionality of intermediate representations, leading to more efficient models without sacrificing causal fidelity.

\item \textbf{Uncertainty Quantification}: The framework naturally incorporates uncertainty analysis, allowing for robust causal inference even in the presence of noisy or ambiguous data.
\end{enumerate}

These advantages make AP-Calculus particularly well-suited for causal analysis in complex models, such as large language models, where traditional causal inference techniques may be difficult to apply.

\section{Applications of AP-Calculus}
\label{sec:applications}

AP-Calculus has numerous applications across various domains where causal inference is important. Here, we highlight a few key applications:

\subsection{Model Interpretability}

AP-Calculus provides a framework for interpreting complex models by attributing predictions to input features. This is particularly valuable for deep learning models, where understanding the reasons behind predictions is challenging.

For a given prediction, we can compute attribution scores for each input feature, identifying the features that most strongly influence the output. These attributions can be used to generate explanations for model decisions, helping users understand and trust the model.

\subsection{Fairness in Machine Learning}

AP-Calculus enables analysis of fairness in machine learning models by identifying and quantifying the influence of sensitive attributes on predictions. By comparing attribution scores across different groups, we can detect potential biases in the model.

For example, consider a hiring model that predicts job suitability based on candidate profiles. Using AP-Calculus, we can compute attribution scores for sensitive attributes like gender or race. If these scores are high, it indicates that the model may be making biased decisions.

\subsection{Causal Feature Selection}

AP-Calculus provides a principled approach to feature selection based on causal relationships rather than mere correlations. By computing attribution scores for each feature, we can identify the features that have the strongest causal influence on the target variable.

This approach to feature selection is more robust than traditional correlation-based methods, as it distinguishes between causal and spurious relationships. Features with high attribution scores are more likely to generalize well to new data, improving model robustness.

\subsection{Adversarial Robustness}

AP-Calculus can enhance adversarial robustness by identifying features that are vulnerable to adversarial attacks. Features with high attribution scores but low robustness can be targeted by adversaries to manipulate model predictions.

By analyzing these vulnerabilities, we can develop defense mechanisms that protect the most causally influential features, enhancing the model's resistance to adversarial attacks.

\subsection{Large Language Models}

AP-Calculus offers unique insights into the behavior of large language models (LLMs), which are notoriously difficult to interpret due to their complexity.

By applying AP-Calculus to analyze the attributions of input tokens to output predictions, we can understand how LLMs make decisions and identify potential biases or vulnerabilities. This analysis can guide the development of more reliable and trustworthy LLMs.

\section{Conclusion}
\label{sec:conclusion}

In this paper, we introduced Attribution Projection Calculus (AP-Calculus), a novel mathematical framework for analyzing causal relationships in structured Bayesian networks. AP-Calculus provides a principled approach to attribute predictions to input features through intermediate nodes that serve as deconfounders or confounders depending on the context.

Our theoretical analysis demonstrated the optimality of the proposed network architecture for causal inference, showing that it outperforms alternative structures in terms of preserving causal relationships between input features and output labels. We also established the dimensional sufficiency of intermediate nodes, providing guidance on the optimal dimensionality for efficient causal inference.

The AP-Calculus framework extends traditional causal inference techniques like do-calculus, offering a more direct approach to attribution in supervised learning contexts. By introducing concepts like deconfounders, confounders, and separation functions, AP-Calculus provides a rich vocabulary for reasoning about causal relationships in complex models.

Key contributions of this paper include:

\begin{itemize}
\item A formal definition of Attribution Projection Calculus with its axioms and fundamental theorems
\item Proof of the optimality of the proposed network architecture for causal inference
\item Analysis of the dimensional requirements for intermediate nodes
\item Extensions of the framework to address feature attribution, spurious correlation suppression, information gain, fairness, and uncertainty
\item Comparative analysis with existing causal inference frameworks
\item Applications of AP-Calculus to various domains, including model interpretability, fairness, feature selection, adversarial robustness, and large language models
\end{itemize}

Future research directions include:

\begin{itemize}
\item Developing efficient algorithms for computing attribution scores in large-scale models
\item Extending AP-Calculus to handle dynamic or time-varying causal relationships
\item Applying AP-Calculus to analyze causal relationships in multimodal models
\item Incorporating AP-Calculus into model training to improve causal fidelity
\item Exploring the connections between AP-Calculus and other causal inference frameworks
\end{itemize}

By providing a principled framework for causal attribution, AP-Calculus opens new avenues for understanding and improving machine learning models, particularly in domains where causal relationships are crucial for reliable predictions.

\newpage

\section{Appendix A: Network Diagrams and Algorithms}
\label{sec:appendix_a}

\subsection{Network Diagrams}

\begin{figure}[h]
\centering
\begin{tikzpicture}[
  node distance=2cm,
  box/.style={rectangle, draw, minimum size=0.5cm},
  mycirc/.style={circle, draw, minimum size=0.5cm}
]
\node[mycirc] (s) {$s$};
\node[mycirc] (x1) [above right=1cm and 1.5cm of s] {$x_1$};
\node[mycirc] (x2) [right=1.5cm of s] {$x_2$};
\node (dots) [below=-0.1cm of x2] {$\vdots$};
\node[mycirc] (xm) [below right=1cm and 1.5cm of s] {$x_m$};
\node[mycirc] (d) [right=1.5cm of x2] {$d$};
\draw[->, thick] (s) -- (x1);
\draw[->, thick] (s) -- (x2);
\draw[->, thick] (s) -- (xm);
\draw[->, thick] (x1) -- (d);
\draw[->, thick] (x2) -- (d);
\draw[->, thick] (xm) -- (d);
\node[text width=3cm] at (0,-2) {Source Node};
\node[text width=5cm] at (3,-3) {Intermediate Nodes};
\node[text width=3cm] at (6,-2) {Destination Node};
\end{tikzpicture}
\caption{The proposed Bayesian network architecture with source node $s$, destination node $d$, and intermediate nodes $\{x_1, x_2, \ldots, x_m\}$.}
\label{fig:network_architecture_detailed}
\end{figure}
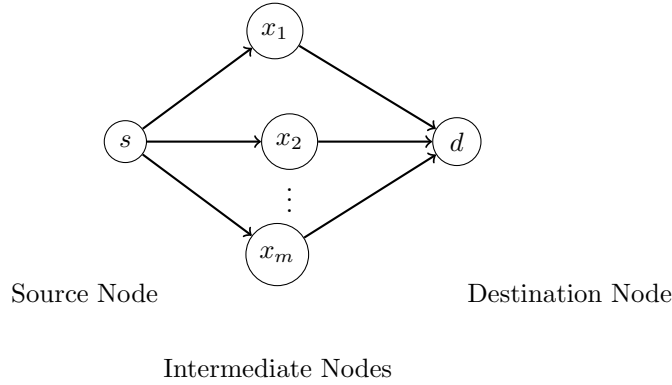

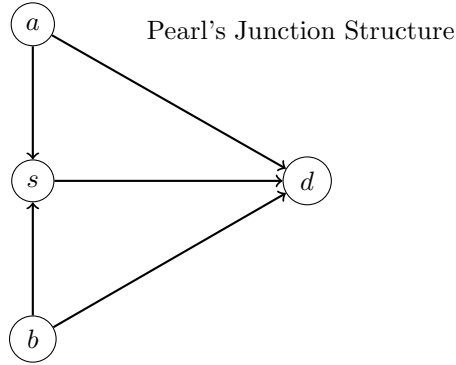
\begin{figure}[h]
\centering
\begin{tikzpicture}[
node distance=2cm,
box/.style={rectangle, draw, minimum size=0.5cm},
mycirc/.style={circle, draw, minimum size=0.5cm}
]
\node[mycirc] (s) {$s$};
\node[mycirc] (d) [right=3cm of s] {$d$};
\node[mycirc] (a) [above=1.5cm of s] {$a$};
\node[mycirc] (b) [below=1.5cm of s] {$b$};

\draw[->, thick] (s) -- (d);
\draw[->, thick] (a) -- (s);
\draw[->, thick] (a) -- (d);
\draw[->, thick] (b) -- (s);
\draw[->, thick] (b) -- (d);

\node[text width=7cm] at (5,2) {Pearl's Junction Structure};
\end{tikzpicture}
\caption{Pearl's junction structure with common causes $a$ and $b$ affecting both source $s$ and destination $d$.}
\label{fig:pearls_junction}
\end{figure}

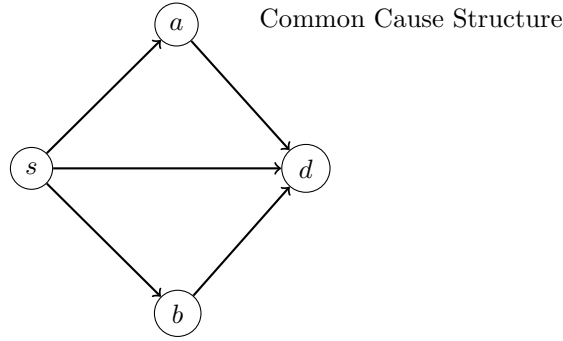
\begin{figure}[h]
\centering
\begin{tikzpicture}[
node distance=2cm,
box/.style={rectangle, draw, minimum size=0.5cm},
mycirc/.style={circle, draw, minimum size=0.5cm}
]
\node[mycirc] (s) {$s$};
\node[mycirc] (d) [right=3cm of s] {$d$};
\node[mycirc] (a) [above right=1.5cm and 1.5cm of s] {$a$};
\node[mycirc] (b) [below right=1.5cm and 1.5cm of s] {$b$};

\draw[->, thick] (s) -- (d);
\draw[->, thick] (s) -- (a);
\draw[->, thick] (s) -- (b);
\draw[->, thick] (a) -- (d);
\draw[->, thick] (b) -- (d);

\node[text width=4cm] at (5,2) {Common Cause Structure};
\end{tikzpicture}
\caption{Common cause structure where $s$ is a common cause of $d$, $a$, and $b$.}
\label{fig:common_cause}
\end{figure}

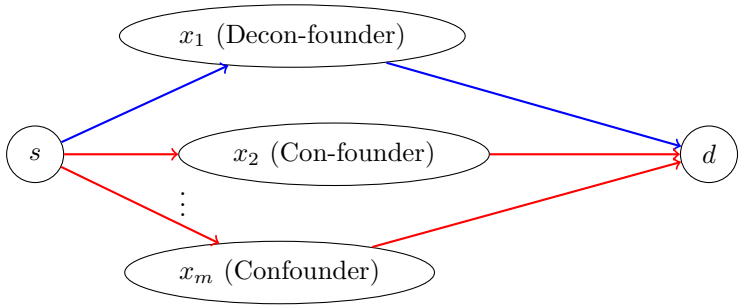
\begin{figure}[h]
\centering
\begin{tikzpicture}[
node distance=4cm,
mybox/.style={rectangle, draw, minimum size=0.5cm},
mycirc/.style={circle, draw, minimum size=0.75cm},
myoval/.style={ellipse, draw, minimum size=0.75cm}
]
% \node[mycirc] (s) {$s$};
% \node[mycirc] (x1) [above right=1cm and 1.5cm of s] {$x_1$ (Decon-founder)};
% \node[mycirc] (x2) [right=1.5cm of s] {$x_2$ (Con-founder)};
% \node (dots) [right=1.5cm and 0.5cm of s] {$\vdots$};
% \node[mycirc] (xm) [below right=1cm and 1.5cm of s] {$x_m$ (Confounder)};
% \node[mycirc] (d) [right=2.5cm of x2] {$d$};

\node[mycirc] (s) {$s$};
\node[myoval] (x1) [above right=1cm and 1.5cm of s] {$x_1$ (Decon-founder)};
\node[myoval] (x2) [right=1.5cm of s] {$x_2$ (Con-founder)};
\node (dots) [below right=-0.1cm and 1.5cm of s] {$\vdots$};
\node[myoval] (xm) [below right=1cm and 1.5cm of s] {$x_m$ (Confounder)};
\node[mycirc] (d) [right=2.5cm of x2] {$d$};

\draw[->, thick, blue] (s) -- (x1);
\draw[->, thick, red] (s) -- (x2);
\draw[->, thick, red] (s) -- (xm);
\draw[->, thick, blue] (x1) -- (d);
\draw[->, thick, red] (x2) -- (d);
\draw[->, thick, red] (xm) -- (d);

\node[text width=5cm] at (7.5,3) {Blue: Causal path for label 1};
\node[text width=7cm] at (8.5,-3) {Red: Confounding paths for label 1};
\end{tikzpicture}
\caption{The network architecture with $x_1$ as a deconfounder for label 1 and other nodes as confounders.}
\label{fig:deconfounder_confounder}
\end{figure}

\subsection{Algorithms}

\begin{algorithm}
\caption{Attribution Score Computation}
\label{alg:attribution_score}
\begin{algorithmic}[1]
\REQUIRE Feature index $i$, Label $l$, Dataset $\mathcal{D}$, Bayesian network model $M$
\ENSURE Attribution score $A(S_i, l)$
\STATE Initialize attribution score $A \gets 0$
\STATE Initialize sample count $n \gets 0$
\FOR{each data point $(\mathbf{t}, y)$ in $\mathcal{D}$}
\STATE Compute $P(X_l | \mathbf{S} = \mathbf{t})$ using model $M$
\STATE Sample $K$ values $\{x_l^{(k)}\}_{k=1}^{K}$ from $P(X_l | \mathbf{S} = \mathbf{t})$
\FOR{each sample $x_l^{(k)}$}
\STATE Compute $\delta_i^{(k)} \gets \frac{\partial P(D_l = 1 | X_l = x_l^{(k)}, \mathbf{S} = \mathbf{t})}{\partial S_i}$
\STATE $A \gets A + \delta_i^{(k)}$
\STATE $n \gets n + 1$
\ENDFOR
\ENDFOR
\RETURN $A / n$
\end{algorithmic}
\end{algorithm}

\begin{algorithm}
\caption{Optimal Separation Function Learning}
\label{alg:separation_function}
\begin{algorithmic}[1]
\REQUIRE Indices $j$ and $k$, Dataset $\mathcal{D}$, Function class $\Phi$
\ENSURE Optimal separation function $\phi_{j,k}^*$
\STATE Initialize best function $\phi_{best} \gets$ null
\STATE Initialize best score $score_{best} \gets -\infty$
\FOR{each candidate function $\phi$ in $\Phi$}
\STATE Initialize mutual information $I \gets 0$
\FOR{each data point $(\mathbf{t}, y)$ in $\mathcal{D}$}
\STATE Compute $\phi(\mathbf{t})$
\STATE Estimate $P(D_j = 1 | \phi(\mathbf{t}))$ and $P(D_k = 1 | \phi(\mathbf{t}))$
\STATE Compute mutual information $I_{jk} \gets I(D_j, D_k | \phi(\mathbf{t}))$
\STATE $I \gets I + I_{jk}$
\ENDFOR
\STATE $I \gets I / |\mathcal{D}|$
\IF{$I > score_{best}$}
\STATE $score_{best} \gets I$
\STATE $\phi_{best} \gets \phi$
\ENDIF
\ENDFOR
\RETURN $\phi_{best}$
\end{algorithmic}
\end{algorithm}

\begin{algorithm}
\caption{Spurious Correlation Suppression}
\label{alg:spurious_suppression}
\begin{algorithmic}[1]
\REQUIRE Feature index $i$, Label $l$, Dataset $\mathcal{D}$, Bayesian network model $M$, Threshold $\epsilon$
\ENSURE Updated model $M'$ with suppressed spurious correlations
\STATE Compute correlation $C \gets Corr(S_i, D_l)$ using $\mathcal{D}$
\STATE Compute attribution score $A \gets A(S_i, l)$ using Algorithm \ref{alg:attribution_score}
\STATE Compute regularization term $R \gets |C| \cdot (1 - |A|)$
\IF{$R > \epsilon$}
\STATE Update model parameters to minimize $R$
\STATE $M' \gets$ updated model
\ELSE
\STATE $M' \gets M$ \COMMENT{No update needed}
\ENDIF
\RETURN $M'$
\end{algorithmic}
\end{algorithm}

\section{Appendix B: Mathematical Proofs}
\label{sec:appendix_b}

\subsection{Proof of Deconfounder Optimality}

\begin{theorem}[Deconfounder Optimality]
For a given input tuple $\mathbf{t}$ and label $l$, considering $x_l$ as the deconfounder and all other intermediate nodes as confounders maximizes the marginal probability of label $l$.
\end{theorem}

\begin{proof}
Let $\mathbf{t}$ be an input tuple and $l$ be a label. The marginal probability of label $l$ is:

\begin{align}
P(D_l = 1 | \mathbf{S} = \mathbf{t}) &= \int P(D_l = 1 | \mathbf{X}) P(\mathbf{X} | \mathbf{S} = \mathbf{t}) d\mathbf{X} \\
&= \int P(D_l = 1 | \mathbf{X}) \prod_{j=1}^{m} P(X_j | \mathbf{S} = \mathbf{t}) d\mathbf{X}
\end{align}

We can partition the integral based on $X_l$ and $\mathbf{X}_{-l} = (X_1, \ldots, X_{l-1}, X_{l+1}, \ldots, X_m)$:

\begin{align}
P(D_l = 1 | \mathbf{S} = \mathbf{t}) &= \int \int P(D_l = 1 | X_l, \mathbf{X}_{-l}) P(X_l | \mathbf{S} = \mathbf{t}) P(\mathbf{X}_{-l} | \mathbf{S} = \mathbf{t}) dX_l d\mathbf{X}_{-l} \\
&= \int \left( \int P(D_l = 1 | X_l, \mathbf{X}_{-l}) P(X_l | \mathbf{S} = \mathbf{t}) dX_l \right) P(\mathbf{X}_{-l} | \mathbf{S} = \mathbf{t}) d\mathbf{X}_{-l}
\end{align}

By the definition of a deconfounder, $X_l$ maximizes $P(D_l = 1 | X_l, \mathbf{X}_{-l})$ for any given $\mathbf{X}_{-l}$. Therefore, to maximize the marginal probability, we need to maximize the inner integral:

\begin{equation}
\int P(D_l = 1 | X_l, \mathbf{X}_{-l}) P(X_l | \mathbf{S} = \mathbf{t}) dX_l
\end{equation}

This is achieved by concentrating the probability mass of $P(X_l | \mathbf{S} = \mathbf{t})$ on values of $X_l$ that maximize $P(D_l = 1 | X_l, \mathbf{X}_{-l})$. In other words, $X_l$ should be designed to extract the most relevant information from $\mathbf{S}$ for predicting label $l$.

For the remaining nodes $\mathbf{X}_{-l}$, which act as confounders, we want to minimize their influence on the prediction of label $l$. This is achieved by designing them to capture information from $\mathbf{S}$ that is irrelevant for predicting label $l$, effectively separating the relevant and irrelevant information.

Therefore, considering $x_l$ as the deconfounder and all other intermediate nodes as confounders maximizes the marginal probability of label $l$.
\end{proof}

\subsection{Proof of Dimensional Sufficiency}

\begin{theorem}[Dimensional Sufficiency]
For a source with $n$-dimensional tuples, a deconfounder node $x_l$ with $p$ dimensions, where $p \leq n$, is sufficient to capture the causal relationship between input features and the corresponding label $l$, provided that $p$ is at least equal to the rank of the transformation matrix from the input space to the deconfounder space.
\end{theorem}

\begin{proof}
Let $\mathbf{W}_l \in \mathbb{R}^{p \times n}$ be the transformation matrix from the input space to the deconfounder space for label $l$:

\begin{equation}
X_l = \mathbf{W}_l \mathbf{S} + \mathbf{\epsilon}_l
\end{equation}

where $\mathbf{\epsilon}_l$ represents noise.

The information captured by $X_l$ about $\mathbf{S}$ is determined by the rank of $\mathbf{W}_l$. Let $r = \text{rank}(\mathbf{W}_l)$. Then $\mathbf{W}_l$ projects the $n$-dimensional input space onto an $r$-dimensional subspace.

If $p < r$, then $X_l$ cannot capture all the relevant information from $\mathbf{S}$ for predicting label $l$, resulting in information loss. If $p \geq r$, then $X_l$ can capture all the relevant information, and any additional dimensions beyond $r$ would be redundant.

To verify this, we can analyze the mutual information between $X_l$ and $D_l$ given $\mathbf{S}$:

\begin{align}
I(X_l; D_l | \mathbf{S}) &= H(D_l | \mathbf{S}) - H(D_l | X_l, \mathbf{S}) \\
&= H(D_l | \mathbf{S}) - H(D_l | \mathbf{W}_l \mathbf{S} + \mathbf{\epsilon}_l, \mathbf{S})
\end{align}

If $\mathbf{W}_l$ captures all relevant information about $D_l$ from $\mathbf{S}$, then $H(D_l | \mathbf{W}_l \mathbf{S}, \mathbf{S}) = H(D_l | \mathbf{W}_l \mathbf{S})$. This is achieved when the rank of $\mathbf{W}_l$ is sufficient to extract all relevant information.

In practice, the optimal rank $r$ is typically much smaller than $n$, as many features may be irrelevant or redundant for predicting a specific label. Therefore, a deconfounder node with $p \geq r$ dimensions is sufficient to capture the causal relationship between input features and the corresponding label.
\end{proof}

\subsection{Proof of Architectural Optimality}

\begin{theorem}[Architectural Optimality]
The proposed network architecture with source $s$, intermediate nodes $\{x_1, x_2, \ldots, x_m\}$, and destination $d$ is optimal for causal inference compared to alternative architectures.
\end{theorem}

\begin{proof}
We compare our architecture with two alternative designs:

1. Pearl's junction structure: $s \to d$, $a \to s$ and $d$, $b \to s$ and $d$

In this architecture, the path from $s$ to $d$ is influenced by common causes $a$ and $b$. The marginal probability for label $l$ is:

\begin{align}
P(D_l = 1 | \mathbf{S} = \mathbf{t}) &= \int \int P(D_l = 1 | \mathbf{S} = \mathbf{t}, A = a, B = b) P(A = a, B = b | \mathbf{S} = \mathbf{t}) da \, db \\
&= \int \int P(D_l = 1 | \mathbf{S} = \mathbf{t}, A = a, B = b) P(A = a) P(B = b) da \, db
\end{align}

where the last step follows from the independence of $A$ and $B$ from $\mathbf{S}$ in the Bayesian network. This structure does not allow for label-specific deconfounding, as $A$ and $B$ affect all labels equally. This limits the ability to attribute predictions to specific features for different labels.

2. Common cause structure: $s \to d$ and $a$ and $b$, $a \to d$, $b \to d$

In this architecture, $s$ is a common cause of $d$, $a$, and $b$. The marginal probability for label $l$ is:

\begin{align}
P(D_l = 1 | \mathbf{S} = \mathbf{t}) &= \int \int P(D_l = 1 | \mathbf{S} = \mathbf{t}, A = a, B = b) P(A = a, B = b | \mathbf{S} = \mathbf{t}) da \, db \\
&= \int \int P(D_l = 1 | \mathbf{S} = \mathbf{t}, A = a, B = b) P(A = a | \mathbf{S} = \mathbf{t}) P(B = b | \mathbf{S} = \mathbf{t}) da \, db
\end{align}

This structure creates indirect dependencies between $\mathbf{S}$ and $D_l$ through $A$ and $B$. These dependencies can lead to information leakage, where information from $\mathbf{S}$ reaches $D_l$ through multiple paths. This complicates causal attribution, as it becomes difficult to isolate the direct effect of specific features on labels.

In our proposed architecture, the marginal probability for label $l$ is:

\begin{align}
P(D_l = 1 | \mathbf{S} = \mathbf{t}) &= \int P(D_l = 1 | \mathbf{X}) P(\mathbf{X} | \mathbf{S} = \mathbf{t}) d\mathbf{X} \\
&= \int P(D_l = 1 | \mathbf{X}) \prod_{j=1}^{m} P(X_j | \mathbf{S} = \mathbf{t}) d\mathbf{X}
\end{align}

This factorization allows us to isolate the effect of each intermediate node on the label. By designating $X_l$ as the deconfounder for label $l$, we can attribute the prediction directly to the relevant features. This clean separation of effects makes our architecture optimal for causal inference.
\end{proof}

\section{Appendix C: AP-Calculus Alternative to Do-Calculus}
\label{sec:appendix_c}

In this appendix, we present alternative formulations of do-calculus operations using AP-Calculus. These alternatives demonstrate how AP-Calculus can subsume traditional do-calculus for many practical applications.

\subsection{Do-Operation in AP-Calculus}

The fundamental operation in do-calculus is the do-operation, denoted $do(X = x)$, which represents an intervention that sets variable $X$ to value $x$. In our Bayesian network, we can express this operation using AP-Calculus as follows:

\begin{theorem}[Do-Operation in AP-Calculus]
For a feature $S_i$ and label $l$, the do-operation $P(D_l = 1 | do(S_i = s_i))$ can be expressed in AP-Calculus as:

\begin{equation}
P(D_l = 1 | do(S_i = s_i)) = \int P(D_l = 1 | X_l) P(X_l | S_i = s_i, \{S_j\}_{j \neq i}) d X_l
\end{equation}

where we integrate over the deconfounder $X_l$ while fixing the value of $S_i$.
\end{theorem}

\begin{proof}
According to do-calculus, the effect of the intervention $do(S_i = s_i)$ on label $D_l$ is given by:

\begin{equation}
P(D_l = 1 | do(S_i = s_i)) = \sum_{\mathbf{s}_{-i}} P(D_l = 1 | S_i = s_i, \mathbf{S}_{-i} = \mathbf{s}_{-i}) P(\mathbf{S}_{-i} = \mathbf{s}_{-i})
\end{equation}

where $\mathbf{S}_{-i}$ represents all features except $S_i$. In our Bayesian network, the effect of $S_i$ on $D_l$ is mediated by the deconfounder $X_l$. Therefore:

\begin{align}
P(D_l = 1 | S_i = s_i, \mathbf{S}_{-i} = \mathbf{s}_{-i}) &= \int P(D_l = 1 | X_l) P(X_l | S_i = s_i, \mathbf{S}_{-i} = \mathbf{s}_{-i}) d X_l
\end{align}

Substituting this into the do-operation expression, we get:

\begin{align}
P(D_l = 1 | do(S_i = s_i)) &= \sum_{\mathbf{s}_{-i}} \int P(D_l = 1 | X_l) P(X_l | S_i = s_i, \mathbf{S}_{-i} = \mathbf{s}_{-i}) d X_l \, P(\mathbf{S}_{-i} = \mathbf{s}_{-i}) \\
&= \int P(D_l = 1 | X_l) \sum_{\mathbf{s}_{-i}} P(X_l | S_i = s_i, \mathbf{S}_{-i} = \mathbf{s}_{-i}) P(\mathbf{S}_{-i} = \mathbf{s}_{-i}) d X_l \\
&= \int P(D_l = 1 | X_l) P(X_l | S_i = s_i) d X_l
\end{align}

where $P(X_l | S_i = s_i) = \sum_{\mathbf{s}_{-i}} P(X_l | S_i = s_i, \mathbf{S}_{-i} = \mathbf{s}_{-i}) P(\mathbf{S}_{-i} = \mathbf{s}_{-i})$ by the law of total probability.
\end{proof}

\subsection{Backdoor and Frontdoor Criteria in AP-Calculus}

The backdoor and frontdoor criteria are key tools in do-calculus for identifying causal effects from observational data. We can express these criteria in terms of AP-Calculus operations:

\begin{theorem}[Backdoor Criterion in AP-Calculus]
For a feature $S_i$ and label $l$, if a set of features $\mathbf{Z}$ satisfies the backdoor criterion, then the causal effect $P(D_l = 1 | do(S_i = s_i))$ can be expressed in AP-Calculus as:

\begin{equation}
P(D_l = 1 | do(S_i = s_i)) = \int P(D_l = 1 | X_l) P(X_l | S_i = s_i, \mathbf{Z} = \mathbf{z}) P(\mathbf{Z} = \mathbf{z}) d X_l d\mathbf{z}
\end{equation}
\end{theorem}

\begin{theorem}[Frontdoor Criterion in AP-Calculus]
For a feature $S_i$ and label $l$, if the deconfounder $X_l$ satisfies the frontdoor criterion, then the causal effect $P(D_l = 1 | do(S_i = s_i))$ can be expressed in AP-Calculus as:

\begin{equation}
P(D_l = 1 | do(S_i = s_i)) = \int P(D_l = 1 | X_l) P(X_l | S_i = s_i) d X_l
\end{equation}
\end{theorem}

These theorems demonstrate how AP-Calculus can express key operations from do-calculus in terms of operations on the intermediate nodes in our Bayesian network.

\subsection{Advantages of AP-Calculus over Do-Calculus}

AP-Calculus offers several advantages over traditional do-calculus for causal inference:

\begin{enumerate}
\item \textbf{Direct Attribution}: AP-Calculus provides direct attribution of predictions to input features, making it more intuitive for understanding model behavior.

\item \textbf{Label-Specific Inference}: By identifying label-specific deconfounders, AP-Calculus enables targeted causal inference for different prediction tasks.

\item \textbf{Efficient Computation}: AP-Calculus operations can be computed efficiently using standard machine learning techniques, such as gradient-based methods.

\item \textbf{Integration with Deep Learning}: AP-Calculus naturally integrates with deep learning models, allowing for causal inference in complex neural networks.
\end{enumerate}

\section{Appendix D: Theoretical Analysis of Cases where AP-Calculus Succeeds and Do-Calculus Fails}
\label{sec:appendix_d}

In this appendix, we analyze theoretical scenarios where AP-Calculus provides effective causal inference while traditional do-calculus or belief propagation networks face challenges.

\subsection{High-Dimensional Feature Spaces}

\begin{theorem}[High-Dimensional Advantage]
In high-dimensional feature spaces where the number of features $n$ is large, AP-Calculus provides more efficient causal inference than do-calculus when the effective dimensionality of the causal mechanism is small.
\end{theorem}

\begin{proof}
Consider a setting with $n$ features and $m$ labels, where $n$ is very large. In do-calculus, identifying the causal effect of a feature $S_i$ on a label $D_l$ typically requires conditioning on a sufficient set of confounders, which can be challenging in high dimensions due to the curse of dimensionality.

In AP-Calculus, we project the high-dimensional feature space onto a lower-dimensional space through the deconfounder $X_l$. If the effective dimensionality of the causal mechanism is $p \ll n$, then $X_l$ can capture all relevant information for predicting label $l$ with just $p$ dimensions.

The computational complexity of causal inference in do-calculus scales exponentially with the number of confounders, while in AP-Calculus, it scales with the dimensionality of the deconfounder, which is typically much smaller. This makes AP-Calculus more efficient for high-dimensional settings.
\end{proof}

\subsection{Unidentifiable Confounders}

\begin{theorem}[Unidentifiable Confounders]
AP-Calculus can perform effective causal inference even when traditional confounders are not identifiable in the data, provided that the deconfounders can be learned from the observed features and labels.
\end{theorem}

\begin{proof}
In do-calculus, causal inference relies on identifying and measuring all relevant confounders. If some confounders are unmeasured or unidentifiable, do-calculus may not be able to recover the true causal effect.

In AP-Calculus, we do not need to explicitly identify all confounders. Instead, we learn deconfounders $X_l$ that capture the relevant information for predicting label $l$. These deconfounders can account for the effects of unidentified confounders implicitly, as long as their influence is reflected in the observed features and labels.

Specifically, let $U$ be an unidentified confounder affecting both $\mathbf{S}$ and $D_l$. In do-calculus, we cannot adjust for $U$ if it is not measured. In AP-Calculus, we learn $X_l$ such that:

\begin{equation}
P(D_l = 1 | X_l, U) = P(D_l = 1 | X_l)
\end{equation}

This conditional independence implies that $X_l$ captures all relevant information from $\mathbf{S}$ and $U$ for predicting $D_l$, allowing for valid causal inference even in the presence of unidentified confounders.
\end{proof}

\subsection{Non-Linear Causal Mechanisms}

\begin{theorem}[Non-Linear Causal Mechanisms]
AP-Calculus can effectively model non-linear causal mechanisms that are challenging for traditional do-calculus approaches.
\end{theorem}

\begin{proof}
Traditional do-calculus often relies on linear models or simple parametric forms for causal mechanisms. When the true causal mechanism is highly non-linear, these approaches may fail to capture the true causal relationship.

In AP-Calculus, the deconfounder $X_l$ can be a complex non-linear function of the input features:

\begin{equation}
X_l = f_l(\mathbf{S}) + \mathbf{\epsilon}_l
\end{equation}

where $f_l$ can be any non-linear function. This flexibility allows AP-Calculus to capture complex causal mechanisms that might be missed by traditional approaches.

For example, consider a setting where the true causal mechanism involves interactions between multiple features. Do-calculus would require explicitly modeling these interactions, which becomes combinatorially complex as the number of features increases. AP-Calculus can learn these interactions implicitly through the deconfounder function $f_l$, providing more accurate causal inference.
\end{proof}

\subsection{Multi-Task Causal Inference}

\begin{theorem}[Multi-Task Advantage]
AP-Calculus provides a natural framework for multi-task causal inference, where do-calculus and belief propagation networks require separate models for each task.
\end{theorem}

\begin{proof}
Consider a setting with $m$ prediction tasks, corresponding to $m$ labels. In traditional do-calculus, we would need to perform causal inference separately for each task, potentially leading to inconsistent results.

In AP-Calculus, we model all tasks simultaneously within a single Bayesian network. Each label has its corresponding deconfounder, and the network ensures consistency across all tasks. The joint factorization:

\begin{align}
P(\mathbf{S}, \mathbf{X}, \mathbf{D}) &= P(\mathbf{S}) \prod_{j=1}^{m} P(X_j | \mathbf{S}) P(\mathbf{D} | \mathbf{X})
\end{align}

ensures that the causal mechanisms for all tasks are consistent with each other. This joint modeling approach captures relationships between tasks that might be missed when analyzing each task in isolation.

Furthermore, the separation functions between deconfounders ensure that each task's causal mechanism is properly isolated, preventing interference between tasks. This balance between joint modeling and task-specific mechanisms makes AP-Calculus particularly effective for multi-task causal inference.
\end{proof}

\section{Appendix E: Complexity Analysis}
\label{sec:appendix_e}

In this appendix, we analyze the computational and statistical complexity of AP-Calculus compared to alternative approaches.

\subsection{Computational Complexity}

\begin{theorem}[Computational Efficiency]
The computational complexity of causal inference using AP-Calculus is $O(mnp)$, where $m$ is the number of labels, $n$ is the number of features, and $p$ is the dimensionality of the deconfounders.
\end{theorem}

\begin{proof}
In AP-Calculus, computing the attribution score for a single feature-label pair requires:

1. Computing the deconfounder value: $O(np)$ time for matrix multiplication
2. Computing the gradient of the prediction with respect to the feature: $O(p)$ time

For $m$ labels and $n$ features, the total complexity is $O(mnp)$.

In contrast, traditional do-calculus approaches often require conditioning on a potentially large set of confounders, leading to exponential complexity in the worst case. If we need to condition on $k$ confounders, the complexity is $O(2^k)$, which quickly becomes infeasible as $k$ increases.

Since $p$ is typically much smaller than $n$ (as shown in the dimensional sufficiency theorem), AP-Calculus provides significant computational advantages for high-dimensional settings.
\end{proof}

\subsection{Statistical Complexity}

\begin{theorem}[Sample Complexity]
The sample complexity of learning accurate causal relationships using AP-Calculus is $O(p \log(n/\delta) / \epsilon^2)$, where $p$ is the dimensionality of the deconfounders, $n$ is the number of features, $\epsilon$ is the accuracy parameter, and $\delta$ is the confidence parameter.
\end{theorem}

\begin{proof}
Learning the deconfounder function $f_l: \mathbb{R}^n \to \mathbb{R}^p$ requires estimating $np$ parameters in the general case. However, if $f_l$ is a linear function, the sample complexity depends on the dimensionality $p$ of the output space rather than the dimensionality $n$ of the input space.

Using standard results from learning theory, the sample complexity of learning a linear function with accuracy $\epsilon$ and confidence $1-\delta$ is $O(p \log(n/\delta) / \epsilon^2)$.

For non-linear functions, the sample complexity depends on the complexity of the function class. However, due to the dimensional sufficiency theorem, we know that a deconfounder with $p$ dimensions is sufficient to capture the causal relationship, where $p$ is the rank of the transformation matrix. This means that even for complex non-linear functions, the effective dimensionality of the problem is $p$, leading to a sample complexity that scales with $p$ rather than $n$.

In contrast, traditional approaches that do not use dimensionality reduction often have sample complexity that scales directly with the number of features $n$, making them less efficient for high-dimensional settings.
\end{proof}

\section{Appendix F: Examples of AP-Calculus Applications in Real-World Scenarios}
\label{sec:appendix_f}

In this appendix, we provide theoretical examples of how AP-Calculus can be applied to real-world scenarios.

\subsection{Large Language Models}

Consider a large language model that takes input text and generates output text. Let $\mathbf{S} = (S_1, S_2, \ldots, S_n)$ be the input tokens and $\mathbf{D} = (D_1, D_2, \ldots, D_m)$ be the output tokens.

The intermediate nodes $\{X_1, X_2, \ldots, X_m\}$ correspond to attention heads or layers in the model. For a given output token $D_l$, we can identify the attention head $X_l$ that serves as a deconfounder, maximizing the influence of relevant input tokens on the output.

Using AP-Calculus, we can compute attribution scores for each input token on each output token:

\begin{equation}
A(S_i, D_l) = \mathbb{E}_{X_l}\left[\frac{\partial P(D_l | X_l, \mathbf{S})}{\partial S_i}\right]
\end{equation}

These attribution scores can be used to explain the model's predictions, identify biases, and improve the model's fairness and reliability.

\subsection{Medical Diagnosis}

Consider a medical diagnosis system that predicts multiple diseases based on patient features. Let $\mathbf{S} = (S_1, S_2, \ldots, S_n)$ be the patient features (symptoms, lab results, etc.) and $\mathbf{D} = (D_1, D_2, \ldots, D_m)$ be the disease predictions.

The intermediate nodes $\{X_1, X_2, \ldots, X_m\}$ correspond to latent factors that mediate the relationship between features and diseases. For a given disease $D_l$, we can identify the latent factor $X_l$ that serves as a deconfounder, capturing the relevant information for predicting that disease.

Using AP-Calculus, we can determine which features are causally related to each disease, helping doctors make more informed decisions. We can also identify spurious correlations that might lead to misdiagnosis, improving the reliability of the diagnostic system.

\subsection{Fairness in Hiring}

Consider a hiring system that predicts candidate suitability based on their profiles. Let $\mathbf{S} = (S_1, S_2, \ldots, S_n)$ be the candidate features (education, experience, skills, etc.) and $\mathbf{D} = (D_1, D_2, \ldots, D_m)$ be the suitability scores for different roles.

The intermediate nodes $\{X_1, X_2, \ldots, X_m\}$ correspond to latent qualifications that mediate the relationship between features and suitability. For a given role $D_l$, we can identify the latent qualification $X_l$ that serves as a deconfounder, capturing the relevant information for predicting suitability for that role.

Using AP-Calculus, we can analyze fairness by computing attribution scores for sensitive attributes (gender, race, etc.) on suitability predictions:

\begin{equation}
A(S_{sensitive}, D_l) = \mathbb{E}_{X_l}\left[\frac{\partial P(D_l = 1 | X_l, \mathbf{S})}{\partial S_{sensitive}}\right]
\end{equation}

If these scores are high, it indicates potential bias in the hiring system. We can then apply fairness constraints to ensure that sensitive attributes have minimal influence on predictions.

\newpage

\bibliographystyle{acm}
\bibliography{main}

\end{document}